\documentclass{article}
\usepackage{ijcai25}

\usepackage{times}
\usepackage{soul}
\usepackage{url}
\usepackage[hidelinks]{hyperref}
\usepackage[utf8]{inputenc}
\usepackage[small]{caption}
\usepackage{graphicx}
\usepackage{amsmath}
\usepackage{amssymb}                 
\usepackage{amsthm}
\usepackage{booktabs}
\usepackage{algorithm}
\usepackage{algorithmic}
\usepackage[switch]{lineno}
\usepackage{xcolor}

\newtheorem{example}{Example}

\usepackage{thmtools}
\usepackage{thm-restate}

\declaretheorem[name=Definition]{definition}

\usepackage{enumitem}
\setlist[1]{itemsep=0.5em}
\setlist{topsep=0.5em}
\setlist{leftmargin=3em}

\usepackage[font=small,labelfont=bf,textfont=rm]{caption}
\usepackage{subcaption}

\graphicspath{{./}}

\newenvironment{pproof}
    {\par\vspace*{0.5ex}\noindent\rm {\bf Proof:} }
    {\par\vspace*{1ex}\hfill\tiny$\Box$}

\makeatletter
\def\input@path{{./}}
\makeatother

\usepackage{xspace}

\DeclareFontFamily{U}{matha}{\hyphenchar\font45}
\DeclareFontShape{U}{matha}{m}{n}{
      <5> <6> <7> <8> <9> <10> gen * matha
      <10.95> matha10 <12> <14.4> <17.28> <20.74> <24.88> matha12
      }{}
\DeclareSymbolFont{matha}{U}{matha}{m}{n}
\DeclareMathSymbol{\odiv}         {2}{matha}{"63}

\makeatletter
\newcommand*{\bigCup}{\mathop{\mathpalette\big@Cup\relax}\slimits@}

\newcommand*{\big@Cup}[2]{%
   \setbox\z@=\hbox{\m@th$#1\Cup$}%
   \setbox\z@=\vtop{\vbox{\kern.2\ht\z@\copy\z@}\kern.1\ht\z@}%
   \setbox\tw@=\hbox{\m@th$#1\bigcup$}%
   \vcenter{\hbox{\resizebox{!}{1.4\ht\tw@}{\box\z@}}}%
}

\makeatother

\newcommand{\twiddle}{\mathrel|\joinrel\sim} 

\newcommand{\mods}[1]{[\![#1]\!]\xspace}
\newcommand{\bel}[1]{[#1]\xspace}
\newcommand{\cbel}[1]{[#1]_{>}\xspace}

\newcommand{\contract}{\div\xspace} 

\newcommand{\STQ}{\oplus_{\mathrm{STQ}}\xspace}

\newcommand{\BetaRevR}[1]{$(\beta{#1}^{\ast}_{\scriptscriptstyle \preccurlyeq})$\xspace}

\newcommand{\CRevR}[1]{$(\mathrm{C}{#1}^{\ast}_{\scriptscriptstyle \preccurlyeq})$\xspace}
\newcommand{\CRevS}[1]{$(\mathrm{C}{#1}_{\scriptscriptstyle\mathrm{b}}^{\ast})$\xspace}

\newcommand{\CRevRn}[1]{$(\mathrm{C}{#1}^{\scriptscriptstyle \circledast}_{\scriptscriptstyle \preccurlyeq})$\xspace}
\newcommand{\CRevSn}[1]{$(\mathrm{C}{#1}^{\scriptscriptstyle \circledast}_{\scriptscriptstyle\mathrm{b}})$\xspace}

\newcommand{\CRevRnPlus}[1]{$(\mathrm{SC}{#1}^{\scriptscriptstyle \circledast}_{\scriptscriptstyle \preccurlyeq})$\xspace}
\newcommand{\CRevSnPlus}[1]{$(\mathrm{SC}{#1}^{\scriptscriptstyle \circledast}_{\scriptscriptstyle\mathrm{b}})$\xspace}

\newcommand{\PCRevRn}[1]{$(\mathrm{PC}{#1}^{\scriptscriptstyle \circledast}_{\scriptscriptstyle \preccurlyeq})$\xspace}
\newcommand{\PCRevSn}[1]{$(\mathrm{PC}{#1}^{\scriptscriptstyle \circledast}_{\scriptscriptstyle\mathrm{b}})$\xspace}

\newcommand{\CConR}[1]{$(\mathrm{C}{#1}^{\scriptscriptstyle \contract}_{\scriptscriptstyle \preccurlyeq})$\xspace}
\newcommand{\CConRn}[1]{$(\mathrm{C}{#1}^{\scriptscriptstyle \odiv}_{\scriptscriptstyle \preccurlyeq})$\xspace}

\newcommand{\CConS}[1]{$(\mathrm{C}{#1}_{\scriptscriptstyle\mathrm{b}}^{\scriptscriptstyle \contract})$\xspace}
\newcommand{\CConSn}[1]{$(\mathrm{C}{#1}_{\scriptscriptstyle\mathrm{b}}^{\scriptscriptstyle \odiv})$\xspace}

\newcommand{\IndRevR}{$(\mathrm{Ind}^{\ast}_{\scriptscriptstyle \preccurlyeq})$\xspace}
\newcommand{\IndRevS}{$(\mathrm{Ind}_{\scriptscriptstyle\mathrm{b}}^{\ast})$\xspace}
\newcommand{\IndRevRn}{$(\mathrm{Ind}^{\circledast}_{\scriptscriptstyle \preccurlyeq})$\xspace}
\newcommand{\IndRevSn}{$(\mathrm{Ind}^{\circledast}_{\scriptscriptstyle\mathrm{b}})$\xspace}

\newcommand{\KRev}[1]{$(\mathrm{K}{#1}^{\ast})$\xspace}
\newcommand{\KRevP}[1]{$(\mathrm{K}{#1}^{\circledast})$\xspace}

\newcommand{\KConP}[1]{$(\mathrm{K}{#1}^{\scriptscriptstyle \odiv})$\xspace}

\newcommand{\HI}{$(\mathrm{HI})$\xspace}

\newcommand{\HIP}{$(\mathrm{HI^{\circ}})$\xspace}
\newcommand{\LIP}{$(\mathrm{LI^{ \circ}})$\xspace}

\newcommand{\LI}{$(\mathrm{LI})$\xspace}

\newcommand{\MultiConAggregPrec}{$(\mathrm{Agg}^{\odiv}_{\scriptscriptstyle \preccurlyeq})$\xspace}

\newcommand{\MultiRevAggregPrec}{$(\mathrm{Agg}^{\circledast}_{\scriptscriptstyle \preccurlyeq})$\xspace}

\newcommand{\LBO}{$(\mathrm{LB}_{\scriptscriptstyle\mathrm{C}}^{\scriptscriptstyle \oplus})$\xspace}

\newcommand{\UBO}{$(\mathrm{UB}_{\scriptscriptstyle\mathrm{C}}^{\scriptscriptstyle \oplus})$\xspace}

\newcommand{\PPAR}{$(\mathrm{PAR}^{\scriptscriptstyle \oplus}_{\scriptscriptstyle \preccurlyeq})$\xspace}
\newcommand{\PPARMin}{$(\mathrm{PAR}^{\scriptscriptstyle \oplus}_{\scriptscriptstyle \min})$\xspace}

\newcommand{\FacPref}{$(\mathrm{F}^{\oplus}_{\preccurlyeq})$\xspace}

\newcommand{\SPU}{$(\mathrm{SPU}^{\oplus}_{\preccurlyeq})$\xspace}

\newcommand{\WPU}{$(\mathrm{WPU}^{\oplus}_{\preccurlyeq})$\xspace}

\newcommand{\FacMin}{$(\mathrm{F}^{\oplus}_{\scriptscriptstyle \min})$\xspace}

\newcommand{\Cn}{\mathrm{Cn}\xspace}

\newcommand{\CRat}{\mathrm{Cl_{rat}}\xspace}

\newcommand{\astN}{\ast_{\mathrm{N}}\xspace}
\newcommand{\astR}{\ast_{\mathrm{R}}\xspace}
\newcommand{\astL}{\ast_{\mathrm{L}}\xspace}

\newcommand{\MultiRevInter}{$(\mathrm{Conj}^{\scriptscriptstyle \circledast})$\xspace}

\newcommand{\EssRev}{$(\mathrm{S}^{\scriptscriptstyle \circledast}_{\scriptscriptstyle\mathrm{b}})$\xspace}

\newcommand{\PeeRev}{$(\mathrm{P}^{\scriptscriptstyle \circledast}_{\scriptscriptstyle\mathrm{b}})$\xspace}

\begin{document}

\title{Parallel Belief Revision via\\ Order Aggregation}

\author{
Jake Chandler$^1$
\and
Richard Booth$^2$\\
\affiliations
$^1$La Trobe University\\
$^2$Cardiff University\\
\emails
jacob.chandler@latrobe.edu.au,
boothr2@cardiff.ac.uk
}

\date{}

\maketitle

\begin{abstract}
Despite efforts to better understand the constraints that operate on single-step parallel (aka ``package'', ``multiple'') revision, very little work has been carried out on how to extend the model to the iterated case. A recent paper by Delgrande \& Jin outlines a range of relevant rationality postulates. While many of these are plausible, they lack an underlying unifying explanation. We draw on recent work on iterated parallel contraction to offer a general method for extending serial iterated belief revision operators to handle parallel change. This method, based on a family of order aggregators known as TeamQueue aggregators, provides a principled way to recover the independently plausible properties that can be found in the literature, without yielding the more dubious ones.
\end{abstract}

\section{Introduction}

The traditional operations of belief revision theory--revision and contraction--were initially designed to take single sentences as inputs, offering a model of ``serial'' change, in which beliefs were incorporated into or excised from an agent's worldview one at a time. In more recent years, these operators have been generalised to handle entire {\em sets} of sentences as inputs, yielding a model that can accommodate ``parallel'' change, where multiple beliefs are simultaneously processed. 

Until  recently, existing work in this direction had focused almost exclusively on {\em single-step} change, i.e.~studying the effects of a single episode of change on an agent's set of beliefs. With only a few exceptions, no research had been carried out on the more general issue of {\em iterated} change, i.e. studying the effects of a sequence of changes.  In relation to iterated parallel contraction, these exceptions include the work of Spohn  \cite{SpohnPC},  who offers a treatment in terms of his ranking-theoretic construction, and a recent paper by Chandler \& Booth \cite{chandler2025parallelbeliefcontractionorder} in which they propose an axiomatically characterised  construction based on a generalisation of their ``TeamQueue'' method of order aggregation (see  \cite{DBLP:journals/ai/BoothC19}). Regarding  iterated parallel revision, the only discussions that we are aware of are those of Zhang \cite{10.1007/978-3-540-24609-1_27} and  Delgrande \& Jin \cite{DelgrandeJames2012PbrR}. Zhang introduces generalisations to the parallel case of a number of well-known principles for iterated serial revision. Delgrande \& Jin critique Zhang's postulates, finding fault in one key principle, and offer several new ones of their own.  As we shall see, however, at least one of the axioms that  Delgrande \& Jin propose is not compelling, and those that are have yet to be underpinned by a convincing construction. 

We take cue from the TeamQueue approach to iterated parallel contraction to offer a similar constructive approach to the case of revision. We find that this approach precisely allows us to derive those principles of Delgrande \& Jin that are plausible, while failing to allow us to derive those that are not. 

The remainder of the paper proceeds as follows. Section \ref{sec:PrincBelCh} briefly recapitulates existing work on serial belief revision, while Section \ref{sec:PrincParaBelCh} does the same for parallel revision, focusing on Delgrande \& Jin's contribution. Section \ref{sec:TQ}  provides a succinct summary of Chandler \& Booth's aggregation-based proposal for parallel contraction. Section \ref{sec:AggRev} offers a somewhat similar aggregation-based solution to our problem of interest. Finally, Section \ref{sec:Concl} summarises the discussion and makes several suggestions for future research. 
Proofs of the various theorems and propositions are included in the appendix.

\section{Background on serial belief revision} 
\label{sec:PrincBelCh}


In what follows, the state of mind of an agent will be represented by an abstract {\em belief state} $\Psi$, which we do not assume to have any particular internal structure. This state gives rise to a {\em belief set} $\bel{\Psi}$ which contains all and only those sentences that the agent takes to be true when in state $\Psi$. Belief sets are  deductively closed and drawn from a  propositional, truth-functional, finitely-generated language $L$. We denote by $\mathrm{Cn}(S)$ the set of classical logical consequences of $S\subseteq L$. Where $A\in L$, we write $\Cn(A)$ instead of $\Cn(\{A\})$. The set of propositional worlds or valuations will be denoted by $W$, and the set of models of a given sentence $A$ by $\mods{A}$.

The standard ``serial''  model consists of two belief change operations, serial revision $\ast$ and contraction $\contract$. These take a state and a single input sentence and return a new state. Revision captures how an agent incorporates the input into their beliefs, while contraction captures the way in which they remove it from them. The model originally dealt with single-step serial change--the change induced by a single application of revision or contraction by a single sentence--but later research turned to iterated serial change--the change induced by a sequence of applications of serial revision or contraction.





%

 The AGM postulates for revision (see \cite{alchourron1985logic}) provide popular rationality constraints on {\em single-step} serial revision: 
%
%
%
%
%
%
%
%
%

%
%
%
%
%
%
%
%

\begin{tabbing}
\=BLAHBL\=\kill

\> \KRev{1} \> $\textrm{Cn}([\Psi\ast A])\subseteq [\Psi\ast A]$ \\[0.1cm]

\> \KRev{2} \> $A\in [\Psi\ast  A]$\\[0.1cm]

\> \KRev{3}  \> $[\Psi\ast  A]\subseteq\textrm{Cn}([\Psi]\cup\{ A\})$\\[0.1cm]

\> \KRev{4} \> If $\neg A\notin [\Psi]$, then $\textrm{Cn}([\Psi]\cup\{ A\})\subseteq[\Psi\ast  A]$\\[0.1cm]

\> \KRev{5} \>  If $A$ is consistent, then so too is $[\Psi\ast A]$\\[0.1cm]

\> \KRev{6} \> If $\textrm{Cn}(A)=\textrm{Cn}(B)$, then $[\Psi\ast A]=[\Psi\ast B]$\\[0.1cm]

\> \KRev{7} \> $[\Psi\ast A\wedge B]\subseteq\textrm{Cn}([\Psi\ast A]\cup\{B\})$\\[0.1cm]

\> \KRev{8} \> If $\neg B\notin [\Psi\ast A]$, then $\textrm{Cn}([\Psi\ast A]\cup\{B\})\subseteq$\\
\> \> $ [\Psi\ast A\wedge B]$\\[-0.25em]
\end{tabbing} 

\vspace{-1em}

\noindent They are derivable from an analogous set of postulates for contraction by means of an equality known as the Levi Identity \cite{levi1977subjunctives}, given by:

\begin{tabbing}
 \=BLAHBL\=\kill
\> \LI \> $\bel{\Psi\ast A} = \Cn(\bel{\Psi \contract \neg A}\cup \{A\})$ \\ 
[-0.25em]
\end{tabbing} 
\vspace{-1em}

\noindent This principle is justified as follows: The simplest way to modify $\bel{\Psi}$ to include $A$ would be to take the closure of the union of $\bel{\Psi}$ and $ \{A\}$, i.e.~ $\Cn(\bel{\Psi}\cup \{A\})$. Doing this, however, would lead to a violation of \KRev{5}, since $\bel{\Psi}$ and $A$ needn't be jointly consistent, even when $A$ is. To ensure consistency without jettisoning more beliefs than required, we therefore  consider the union of  $\bel{\Psi\contract \neg A}$ and $\{A\}$ instead. 

The so-called Harper Identity \cite{harper1976rational} takes us in the other direction, from single-step serial revision to contraction:

\begin{tabbing}
 \=BLAHBL\=\kill
\>\HI   \> $\bel{\Psi\contract A}= \bel{\Psi} \cap \bel{\Psi\ast \neg A}$ \\[-0.25em]
\end{tabbing} 
\vspace{-1em}


%

\noindent It has been shown that the single-step behaviour of a serial revision operator $\ast$ that satisfies the above postulates can be represented by associating with each state $\Psi$ a reflexive, complete and transitive binary relation (aka total preorder or TPO) $\preccurlyeq_\Psi$ over $W$, such that $\mods{\bel{\Psi}} = \min(\preccurlyeq_\Psi, W)$ and $\min(\preccurlyeq_{\Psi\ast A},W)=\min(\preccurlyeq_{\Psi}, \mods{A})$ (see  \cite{katsuno1991propositional}). We use $x \sim_{\Psi} y$ when $x \preccurlyeq_\Psi y$ and  $y \preccurlyeq_\Psi x$,  and $x \prec_{\Psi} y$ when $x \preccurlyeq_\Psi y$ but  $y \not\preccurlyeq_\Psi x$. Equivalence classes of worlds, i.e. sets of worlds closed under $ \sim_{\Psi} $, are sometimes represented using curly brackets, so that we write ``$x\prec\{y,z\}$'' instead of ``$x\prec y$ and $y\sim z$''. We  take the AGM postulates, and hence the TPO representability of single-step change, for granted. We call an operator that satisfies the AGM postulates an ``AGM operator''.

This single-step behaviour can also be captured by (i) enriching the language with a conditional connective $>$, associating each state $\Psi$ with a conditional belief set  $\cbel{\Psi}:=\{A>B\mid B\in \bel{\Psi\ast A}\}$ or (ii) associating $\Psi$ with a nonmonotonic consequence relation  $\twiddle_\Psi=\{\langle A, B\rangle\mid A > B\in \cbel{\Psi} \}$.  In (ii),  the AGM postulates ensure that  $\twiddle_\Psi$ is ``rational'' (in the sense of \cite{lehmann1992does}) and ``consistency preserving'' (see \cite{10.1007/BFb0018421}).



Many of the principles that follow can be presented in multiple equivalent ways. When this occurs, we use subscripts to distinguish between these formulations, dropping the subscript to more refer to the principle beyond specific presentation. A principle framed in terms of TPOs is denoted using  a subscript ``$\preccurlyeq$''. Occasionally, we may also want to express principles via minimal sets, symbolizing the $\preccurlyeq$-minimal subset of $S\subseteq W$, defined as $\{x\in S\mid \forall y\in S, x\preccurlyeq y\}$, by $\min(\preccurlyeq, S)$. The subscript ``$\min$'' denotes presentation in this style. When appropriate, the names of principles presented in terms of belief sets will include the subscript ``$\mathrm{b}$''. In addition to subscripts, we also use superscripts to remind the reader of the nature of the operation that is constrained by the relevant principle, so that, for example, some principles will carry the superscripts ``$\ast$'' or ``$\contract$''. 


In \cite{darwiche1997logic},  Darwiche \& Pearl proposed a set of postulates (henceforth the ``DP postulates'') governing {\em sequences} of serial revisions. These can be presented either ``syntactically'' in terms of belief sets or ``semantically'' in terms of TPOs. 
%
%
%
%
%
%
%
%
%
In syntactic terms, we have: 
\begin{tabbing}
\=BLAHBL\=\kill
\> \CRevS{1}  \> If $A \in \Cn{(B)}$ then $\bel{(\Psi \ast A) \ast B} = \bel{\Psi \ast B}$  \\[0.1cm]

\> \CRevS{2}  \> If $\neg A \in \Cn{(B)}$ then $\bel{(\Psi \ast A) \ast B} = \bel{\Psi \ast B}$ \\[0.1cm]

\> \CRevS{3}  \> If $A \in \bel{\Psi \ast B}$ then $A \in \bel{(\Psi \ast A) \ast B}$  \\[0.1cm]

\> \CRevS{4}  \> If $\neg A \not\in \bel{\Psi \ast B}$ then $\neg A \not\in \bel{(\Psi \ast A) \ast B}$\\[-0.25em]
\end{tabbing} 
\vspace{-1em}
Semantically, they are given by:
\begin{tabbing}
\=BLAHBL\=\kill

\>\CRevR{1}  \> If $x,y \in \mods{A}$ then $x \preccurlyeq_{\Psi \ast A} y$ iff  $x \preccurlyeq_\Psi y$\\[0.1cm]
\>\CRevR{2} \> If $x,y \in \mods{\neg A}$ then $x \preccurlyeq_{\Psi \ast A} y$ iff  $x \preccurlyeq_\Psi y$\\[0.1cm]

\> \CRevR{3}  \> If $x \in \mods{A}$, $y \in \mods{\neg A}$ and $x \prec_\Psi y$ then $x \prec_{\Psi \ast A}y$ \\[0.1cm]

\> \CRevR{4}  \> If $x \in \mods{A}$, $y \in \mods{\neg A}$ and $x \preccurlyeq_\Psi y$ then $x \preccurlyeq_{\Psi \ast A} y$\\[-0.25em]
\end{tabbing} 
\vspace{-1em}

\noindent We shall call an operator that satisfies the DP postulates a ``DP operator''.

Beyond these,  \cite{booth2006admissible} introduced  a strengthening of both \CRevS{3} and \CRevS{4} which they called ``(P)'' and which later appeared in \cite{jin2007iterated} under the name of ``Independence'': 
\begin{tabbing}
\=BLAHBL\=\kill
\> \IndRevS \> If $\neg A \not\in \bel{\Psi \ast B}$, then $A \in \bel{(\Psi \ast A) \ast B}$ \\[-0.25em]
\end{tabbing} 
\vspace{-1em}
Its semantic counterpart is given by:
\begin{tabbing}
\=BLAHBL\=\kill
\>  \IndRevR \> If $x \in \mods{A}$, $y \in \mods{\neg A}$ and $x \preccurlyeq_{\Psi} y$, then $ x \prec_{\Psi \ast A} y$ \\[-0.25em]
\end{tabbing} 
\vspace{-1em}

\noindent In semantic terms, the previous postulates only constrain the relation between the prior TPO and the TPO resulting from revision by a given sentence. In \cite{DBLP:journals/ai/BoothC20} two further postulates, \BetaRevR{1} and  \BetaRevR{2}, were introduced to constrain the relation between different posterior TPOs resulting from revisions by different sentences.

Constructive proposals premised on the idea that belief states can be identified with TPOs have also been tabled (though see \cite{DBLP:journals/jphil/BoothC17} for criticism). These include most notably  the operations of lexicographic ($\astL$), restrained ($\astR$)  and natural ($\astN$) revision (see, respectively, \cite{nayak2003dynamic},  \cite{booth2006admissible} and \cite{boutilier1996iterated}). They all satisfy the DP postulates, as well as  \BetaRevR{1} and  \BetaRevR{2}. Lexicographic and restrained revision additionally satisfy \IndRevS. Natural revision does not.

Although  \CRevS{1},  \CRevS{3} and  \CRevS{4} are fairly uncontroversial,  \CRevS{2} has received some criticism, which is relevant to the discussion of parallel revision that follows. Here is an alleged counterexample, due to Konieczny \& Pino P\'erez \cite{doi:10.1080/11663081.2000.10511003}. 

\begin{example}[Konieczny \& Pino P\'erez \cite{doi:10.1080/11663081.2000.10511003}] \label{ex:KPPRevise}
Consider a circuit containing an adder and a multiplier. We initially have no information about the working condition of either component. We then come to believe of each that they are working. We then change our mind again and believe that the multiplier is not working after all. After this second change of mind, the thought goes, it should not be the case that we also lose our belief that the adder is working.
\end{example}

\noindent Let $A$ and $B$ respectively stand for the claims that the adder is in order and that the multiplier is. Let $\Psi$ be the initial belief  state. If the initial change in view is to be modelled as a revision by a single sentence,  it would appear that we move from $\Psi$ to state $\Psi\ast A\wedge B$ and then $(\Psi\ast A\wedge B)\ast \neg B$. Since $\neg(A\wedge B)\in \Cn(\neg B)$ and $A\notin \bel{\Psi}$, the alleged intuition that $A\in\bel{(\Psi\ast A\wedge B)\ast \neg B}$ would mean that Example \ref{ex:KPPRevise} contradicts \CRevS{2}. 

Chopra {\em et al}, who take this kind of example to mean that  \CRevS{2}  is not acceptable,
 recommend replacing \CRevS{2} with:

\begin{tabbing}
\=BLAHBL\=\kill

\> (GR$_{\scriptscriptstyle\mathrm{b}}^{\scriptscriptstyle \ast}$) \> $\bel{(\Psi \ast \neg A) \ast A} = \bel{\Psi \ast A}$
%
\\[-0.25em]
\end{tabbing} 
\vspace{-1em}

\noindent We return to Example \ref{ex:KPPRevise} in the next section.

\section{Background on parallel belief revision} 
\label{sec:PrincParaBelCh}

The serial model’s limitations have been highlighted in several previous discussions, particularly its inability to fully encompass the spectrum of potential alterations in belief. It has been suggested that this model should be expanded to include the concepts of parallel revision and contraction, which involve the simultaneous addition or removal of a finite {\em set} $S= \{A_1,\ldots, A_n\}$ of sentences in $L$  (with set of indices $I=\{1,\ldots, n\}$).

We will use $\circledast$ and $\odiv$ to represent parallel revision and contraction, respectively, and assume that when the input is a singleton, the single-step effects of these operations can be expressed in terms of those of their serial equivalents, with $\bel{\Psi\circledast \{A\}} = \bel{\Psi\ast A}$ and $\bel{\Psi\odiv \{A\}} = \bel{\Psi\contract A}$. The symbol $\bigwedge S$ will represent $A_1\wedge \ldots\wedge A_n$ and $\neg S$ will represent $\{\neg A\mid A\in S\}$.

\subsection{Single-step parallel revision} 

For single-step parallel revision, a plausible proposal, which we endorse here, is to simply identify the belief set obtained by parallel revision by $S$ with the belief set obtained by serial revision by the conjunction $\bigwedge S$ of the members of $S$:

\begin{tabbing}
 \=BLAHBL\=\kill
\> \MultiRevInter \> $\bel{\Psi\circledast S} = \bel{\Psi\circledast \{\bigwedge S\}} $ \\[-0.25em]
\end{tabbing} 
\vspace{-1em}

\noindent Given the AGM postulates for serial revision, this suggestion is equivalent to a set of postulates for single-step parallel revision given in \cite{PeppasLimit}, where they are credited, with minor differences, to a preliminary version of \cite{LindstromTech}, published in 1991:

\begin{tabbing}
\=BLAHBL\=\kill

\> \KRevP{1} \> $\textrm{Cn}([\Psi\circledast S])\subseteq [\Psi\circledast S]$ \\[0.1cm]

\> \KRevP{2} \> $S\subseteq [\Psi\circledast  S]$\\[0.1cm]

\> \KRevP{3}  \> $[\Psi\circledast  S]\subseteq\textrm{Cn}([\Psi]\cup S)$\\[0.1cm]

\> \KRevP{4} \> If $[\Psi]\cup S$ is consistent, then $\textrm{Cn}([\Psi]\cup S)\subseteq$\\
\> \> $ [\Psi\circledast  S]$\\[0.1cm]

\> \KRevP{5} \>  If $S$ is consistent, then so is $[\Psi\circledast S]$\\[0.1cm]

\> \KRevP{6} \> If $\textrm{Cn}(S_1)=\textrm{Cn}(S_2)$, then $[\Psi\circledast S_1]=[\Psi\circledast S_2]$\\[0.1cm]

\> \KRevP{7} \> $[\Psi\circledast (S_1\cup S_2)]\subseteq\textrm{Cn}([\Psi\circledast S_1]\cup S_2)$\\[0.1cm]

\> \KRevP{8} \> If $[\Psi\circledast S_1]\cup S_2$ is consistent, then \\
\> \> $ \textrm{Cn}([\Psi\circledast S_1]\cup S_2)\subseteq [\Psi\circledast (S_1\cup S_2)]$\\[-0.25em]
\end{tabbing} 
\vspace{-1em}

\noindent Indeed, as is noted by Peppas \cite{PeppasLimit} and later Delgrande \& Jin \cite{DelgrandeJames2012PbrR}, who  all endorse  \MultiRevInter, the latter is an immediate consequence of \KRevP{6}.  Similarly, assuming \MultiRevInter, \KRevP{1} to \KRevP{8} are obviously recoverable from the AGM postulates.

We noted in Section \ref{sec:PrincBelCh} that, in the single-step situation, a correspondence exists between the AGM postulates for serial revision and contraction, via the Levi and Harper Identities. Since a similar set of postulates to  \KRevP{1}--\KRevP{8}, labelled \KConP{1}--\KConP{8} in  \cite{chandler2025parallelbeliefcontractionorder}, extends the AGM postulates for serial contraction to the parallel case, it is natural to ask whether a two-way correspondence can be established here too. The picture, however, is much less satisfactory here: in \cite[Thm~19.1]{1885-10318}, Fuhrmann showed that, if $\odiv$ satisfies \KConP{1}--\KConP{8} and  $\bel{\Psi \circledast S}$ is defined from $\bel{\Psi \odiv \neg S}$ by the following straightforward generalisation of the Levi Identity 
\begin{tabbing}
 \=BLAHBL\=\kill
\> \LIP \> $\bel{\Psi\circledast S} = \Cn(\bel{\Psi \odiv \neg S}\cup S)$  \\[-0.25em]
\end{tabbing} 
\vspace{-1em}

\noindent then $\circledast$ satisfies \KRevP{1}, \KRevP{2}, \KRevP{3}, \KRevP{4}, \KRevP{7},  and  \KRevP{8}, as well as the following weakening of \KRevP{6}:

\begin{tabbing}
\=BLAHBL\=\kill

\> \KRevP{6-} \> If, $\forall A_1\in S _1$, $\exists A_2\in S _2$ s.t.~$\textrm{Cn}(A_1)=\textrm{Cn}(A_2)$,\\
\> \> and vice versa, then $[\Psi\circledast S _1]=[\Psi\circledast S _2]$\\[-0.25em]

\end{tabbing}

\vspace{-1em}

\noindent Importantly, Fuhrmann notes that \KRevP{5} is {\em not} recoverable, as it clashes with the following plausible principle of ``Disjunctive Persistence'', which states that it is possible to perform a parallel contraction by a consistent set of sentences without thereby removing the belief that at least one of the members of that set is true:

\begin{tabbing}
\=BLAHBL\=\kill

\> (DiP$^{\scriptscriptstyle \odiv}$) \> There exist $\Psi$ and consistent $S\subseteq L$, s.t.\\
\> \> $ \bigvee S\in \bel{\Psi\odiv S}$\\[-0.25em]

\end{tabbing} 

\vspace{-1em}

\noindent Indeed, by this principle, even if $S$ is itself consistent, it may fail to be consistent with $\bel{\Psi\odiv \neg S}$, since we could still have $\bigvee \neg S\in \bel{\Psi\odiv \neg S}$. This leaves $\Cn(\bel{\Psi \odiv \neg S}\cup S)$ inconsistent and hence, by \LIP, $\bel{\Psi\circledast S}$ inconsistent as well. This is prohibited by \KRevP{5}, which requires $\bel{\Psi\circledast S}$ is inconsistent only if $S$ is. This is an important and problematic result, in our view, and we take it to demonstrate the implausibility of this seemingly natural way of extending the Levi Identity to the parallel case.


The fact that one can perform a parallel contraction by a set of sentences without removing the belief that at least one element is true also poses a problem for the most straightforward extension of the Harper Identity, namely:

\begin{tabbing}
 \=BLAHBL\=\kill
\> \HIP   \> $\bel{\Psi\odiv S}= \bel{\Psi} \cap \bel{\Psi\circledast \neg S}$ \\[-0.25em]
\end{tabbing} 
\vspace{-1em}

\noindent Indeed, the fact that it may be the case that $\bigvee S\in \bel{\Psi\odiv S}$ this time leads to a conflict with \KRevP{2}, i.e.~the requirement that $\neg S\subseteq [\Psi\circledast  \neg S]$: from the latter, assuming \KConP{1}, \KConP{5} and consistency of $ \neg S$, we have $\bigvee S\notin \bel{\Psi\circledast \neg S}$ and so, by  \HIP, $\bigvee S\notin \bel{\Psi\odiv S}$.

To sum up, we take \MultiRevInter~to be the correct way to handle the single-step case, with \LIP~and \HIP~proving to be implausible as generalisations of \LI~and \HI. 

\subsection{Iterated parallel revision}

In the single-step case, we can plausibly reduce parallel revision by $S$ to serial revision by $\bigwedge S$, As Delgrande \& Jin (\cite{DelgrandeJames2012PbrR}) have noted, this is not so in the iterated case: while we can identify $\bel{\Psi\circledast\{A, B\}}$ and $\bel{\Psi\ast A\wedge B}$, differences can emerge in subsequent operations, so that there may exist a sentence $C$, such that $\bel{(\Psi\circledast\{A, B\})\ast C}$ and $\bel{(\Psi\ast A\wedge B)\ast C}$ are distinct: the equality $\bel{(\Psi\circledast S)\circledast S'} = \bel{(\Psi\circledast \{\bigwedge S\})\circledast S'} $ may fail. So the beliefs states $\Psi\circledast\{A, B\}$ and $\Psi\ast A\wedge B$ cannot always be equated.
%
%

Interestingly, the example  supporting this claim is none other than Example \ref{ex:KPPRevise} above, which was deployed in criticism of \CRevS{2}.
Delgrande \& Jin argue that the first change in belief in this example (i.e. coming to believe of each of the adder and the multiplier that they are in working order) is not appropriately modelled by revision by a conjunction (so that we move to the state $\Psi\ast A\wedge B$) but rather by a revision by the corresponding set of conjuncts (so that we move to the state $\Psi\circledast \{A, B\}$). If the former model was appropriate, they claim, we would, after ultimately revising by the proposition that the multiplier is not working, lose our belief that the adder is functioning (following \CRevS{2}, which they find appropriate). However, if we properly interpret the situation as  parallel revision, they tell us no such consequence follows: intuitively $A\in \bel{(\Psi\circledast \{A, B\})\circledast\{\neg B\}}$.

This is  plausible and, indeed, the kind of motivation given for  \CRevS{2} does not carry over to the parallel case. As we saw earlier, the intuition there is that the effects of a serial revision should be undone if it is immediately followed by a revision that runs contrary to it. But in the parallel change modelling of Example  \ref{ex:KPPRevise}, the second step runs contrary to only a {\em strict subset} of the initial parallel operations. 

For the above reason, Delgrande \& Jin reject a strong parallel version of \CRevS{2},  proposed by Zhang in  \cite{10.1007/978-3-540-24609-1_27}:\footnote{The nomenclature is ours here, with the ``S'' standing for ``strong'', for contrast with a weaker version below.}
\begin{tabbing}
\=BLAHBL\=\kill

\> \CRevSnPlus{2} \> If $S_1\cup S_2$ is inconsistent, then  $\bel{(\Psi\circledast S_1)\circledast S_2} $\\
\> \> $ = \bel{\Psi\circledast S_2}$\\[0.1cm]

\> \CRevRnPlus{2} \> If $x,y \notin\mods{\bigwedge S}$ then $x \preccurlyeq_{\Psi\circledast S}   y$ iff  $x \preccurlyeq_\Psi y$\\[-0.25em]
\end{tabbing} 
\vspace{-1em}

\noindent since, given this principle, whenever $A\notin\bel{\Psi\circledast \{\neg B\}}$, it would follow that $A\notin \bel{(\Psi\circledast \{A, B\})\circledast\{\neg B\}}$, since $\{A, B\}\cup\{\neg B\}$ is inconsistent, yielding the incorrect result in the parallel revision interpretation of  Example \ref{ex:KPPRevise}. 

If the reductive proposal that we have seen does not plausibly generalise to the iterated case, how exactly, then, are we to handle iterated parallel revision?  Delgrande \& Jin endorse the following generalisations of \CRevR{1}, \CRevR{3} and  \CRevR{4}:

\begin{tabbing}
\=BLAHBL\=\kill

\> \CRevRn{1}\>  If $x,y \in\mods{\bigwedge S}$, then $x \preccurlyeq_{\Psi \circledast S}   y$ iff  $x \preccurlyeq_\Psi y$ \\[0.1cm]


\> \CRevRn{3}   \> If $x \in\mods{\bigwedge S}$, $y \notin\mods{\bigwedge S}$  and $x \prec_\Psi y$, then \\
\> \> $ x \prec_{\Psi \circledast S}   y$ \\[0.1cm]

\> \CRevRn{4} \> If $x \in\mods{\bigwedge S}$, $y \notin\mods{\bigwedge S}$ and $x \preccurlyeq_\Psi y$, then \\
\> \> $ x \preccurlyeq_{\Psi \circledast S}   y$\\[-0.25em]
\end{tabbing} 
\vspace{-1em}

\noindent They tell us (omitting the proof) that these correspond to the following syntactic principles presented in \cite{10.1007/978-3-540-24609-1_27}:

\begin{tabbing}
\=BLAHBL\=\kill

\> \CRevSn{1}\> If $S_1\subseteq \Cn{(S_2)}$, then  $\bel{(\Psi\circledast S_1)\circledast S_2} =$$  \bel{\Psi\circledast S_2}$  \\[0.1cm]

\> \CRevSn{3}   \> If $S_1\subseteq \bel{\Psi\circledast S_2}$, then  $S_1\subseteq \bel{(\Psi\circledast S_1)\circledast S_2} $ \\[0.1cm]

\> \CRevSn{4} \> If $S_1\cup \bel{\Psi\circledast S_2}$ is consistent, then so is  \\
\> \> $ S_1\cup \bel{(\Psi\circledast S_1)\circledast S_2} $ \\[-0.25em]

\end{tabbing} 
\vspace{-1em}

\noindent Delgrande \& Jin do not endorse any kind of parallel version of \CRevS{2} and indeed do not actually consider the question of whether a more plausible alternative to \CRevSnPlus{2} could be found. Such a generalisation, however, is not hard to devise:

\begin{tabbing}
\=BLAHBL\=\kill

\> \CRevRn{2} \> If $x,y \in\mods{\bigwedge \neg S}$, then $x \preccurlyeq_{\Psi\circledast S}   y$ iff  $x \preccurlyeq y$\\[-0.25em]

\end{tabbing} 
\vspace{-1em}

\noindent  A corresponding syntactic form can be given too:

\begin{restatable}{prop}{RPackSyntPlus}
\label{prop:RPackSyntPlus}
Let $\circledast$ be a parallel revision operator such that, for some AGM serial revision operator $\ast$, $\circledast$ and $\ast$ jointly satisfy  
 \MultiRevInter. Then \CRevRn{2} is equivalent to:

\begin{tabbing}
\=BLAHBL\=\kill

\> \CRevSn{2} \> If $\neg S_1\subseteq\Cn{(S_2)}$, then  $\bel{(\Psi\circledast S_1)\circledast S_2} =$$  \bel{\Psi\circledast S_2}$\\[-0.25em]
 
\end{tabbing} 
\vspace{-1em}

\end{restatable}

\noindent Unlike  its stronger counterpart  \CRevSnPlus{2}, this principle  doesn't get us into trouble in relation to Example 1. Indeed, this principle is perfectly consistent with its being the case that both $A\notin\bel{\Psi\circledast \{\neg B\}}$ and $A\in \bel{(\Psi\circledast \{A, B\})\circledast\{\neg B\}}$, since we have $\{\neg A, \neg B\}\nsubseteq\Cn(\{\neg B\})$ and hence the principle doesn't generate the equality $\bel{(\Psi\circledast S_1)\circledast S_2} = \bel{\Psi\circledast S_2}$.

As it turns out, the entire set  \CRevSn{1}--\CRevSn{4} of parallel versions of of the DP postulates is derivable, from two further --strong but plausible--principles that Delgrande \& Jin tentatively endorse.  The syntactic versions of the latter are given by:

\begin{tabbing}
\=BLAHBL\=\kill

\> \PCRevSn{3}  \> If (i) $S_2\neq\varnothing$ and (ii), for all $S\subseteq S_1$ s.t.~$S\cup S_2$\\
\> \>  is consistent, we have  $A\in \bel{\Psi\circledast (S\cup S_2)}$, then \\
\> \> (iii) $A\in \bel{(\Psi\circledast  S_1)\circledast  S_2}$ \\[0.1cm]

\>  \PCRevSn{4} \> If (i) $S_2\neq\varnothing$ and (ii), for all $S\subseteq S_1$ s.t.~$S\cup S_2$\\
\> \>  is consistent, we have $\neg A\notin \bel{\Psi\circledast (S\cup S_2)}$, \\
\> \>  then (iii) $\neg A\notin \bel{(\Psi\circledast  S_1)\circledast  S_2}$ \\[-0.25em] 

\end{tabbing} 
\vspace{-1em}

\noindent While these principles are perhaps a little tricky to interpret, their semantic versions, which are also provided in \cite{DelgrandeJames2012PbrR}, have a much more immediate appeal. They are formulated using the following useful notation:

\begin{definition}
Where $S\subseteq L$, and $x\in W$, $(S\!\mid\! x)$ denotes the subset of $S$ that is true in $x$ (i.e.~$(S\!\mid\! x)=\{A\in S\!\mid\! x\models A\}$).
\end{definition}

\noindent and are given by:

\begin{tabbing}
\=BLAHBL\=\kill

\> \PCRevRn{3}  \> If $(S\!\mid\! y)\subseteq (S\!\mid\! x)$ and $x \prec_\Psi y$ then $x \prec_{\Psi \circledast S}   y$ \\[0.1cm]

\>  \PCRevRn{4} \> If $(S\!\mid\! y)\subseteq (S\!\mid\! x)$ and $x \preccurlyeq_\Psi y$ then $x \preccurlyeq_{\Psi \circledast S}   y$\\[-0.25em]
\end{tabbing} 
\vspace{-1em}

\noindent These principles essentially state the following: if $x$ makes true at least those sentences in $S$ that $y$ does, then revision by $S$ doesn't improve the position of $y$ with respect to $x$. By contrast, \CRevRn{3} and \CRevRn{4} only jointly stipulate that the position of $y$ with respect to $x$ doesn't improve in the special case in which $x$ makes all the sentences in $S$ true and $y$ does not make them all true. So, as Delgrande \& Jin remark, \PCRevRn{3}  and \PCRevRn{4}  entail \CRevRn{3} and \CRevRn{4}. But note that principles \CRevRn{1} and \CRevRn{2} are {\em also} derivable from \PCRevRn{3} and \PCRevRn{4}. Indeed, it obviously follows from the latter that, when $(S\!\mid\! y) = (S\!\mid\! x)$,  we have $x \preccurlyeq_\Psi y$ iff $x \preccurlyeq_{\Psi \circledast S}   y$. But clearly $(S\!\mid\! y) = (S\!\mid\! x)$ holds true whenever the antecedents of either  \CRevRn{1} or  \CRevRn{2} do, i.e. whenever either $x, y \in\mods{\bigwedge S}$ or $x, y \in\mods{\bigwedge \neg S}$. So, to summarise, we have:

\begin{restatable}{prop}{PCtoC}
\label{prop:PCtoC}
Let $\circledast$ be a parallel revision operator that satisfies  \PCRevRn{3} and  \PCRevRn{4}. Then $\circledast$ satisfies   \CRevRn{1}--\CRevRn{4}.

\end{restatable}

\noindent Notably, the problematic \CRevRnPlus{2} does {\em not} follow from these principles, however. We shall  present a construction for which \PCRevRn{3}  and \PCRevRn{4} are sound but  \CRevRnPlus{2} is not. (Proposition  \ref{prop:NoStrongTwo}  and Proposition \ref{prop:PCR} below.)

Besides  \PCRevRn{3}  and \PCRevRn{4}, Delgrande \& Jin  also endorse this plausible postulate, which we give in terms of both belief sets and minimal sets:

\begin{tabbing}
\=BLAHBL\=\kill

\> \EssRev  \>  $ \bel{\Psi\circledast (S_1 \cup S_2)}=$\\
\> \> $  \bel{(\Psi\circledast (S_1 \cup \neg S_2))\circledast  (S_1 \cup S_2)}$ \\[0.1cm]

\> $(\mathrm{S}^{\scriptscriptstyle \circledast}_{\scriptscriptstyle \min})$  \>  $\min(\preccurlyeq_{\Psi}, \mods{\bigwedge(S_1\cup S_2)})=$\\
\> \> $ \min(\preccurlyeq_{\Psi\circledast (S_1\cup \neg S_2)}, \mods{\bigwedge (S_1\cup S_2)})$ \\[-0.25em]

\end{tabbing} 
\vspace{-1em}

\noindent This mildly strengthens the parallel version of the principle (GR$^{\scriptscriptstyle \ast}$), mentioned in relation to  Example \ref{ex:KPPRevise} (set $S_1=\varnothing$):
\begin{tabbing}
\=BLAHBL\=\kill

\> (GR$_{\scriptscriptstyle\mathrm{b}}^{\scriptscriptstyle \circledast}$) \> $\bel{(\Psi \circledast \neg S) \circledast S} = \bel{\Psi \circledast S}$\\[0.1cm]

\> (GR$^{\scriptscriptstyle \circledast}_{\scriptscriptstyle \min}$) \> $\min(\preccurlyeq_{\Psi\circledast \neg S}, \mods{\bigwedge S})=\min(\preccurlyeq_{\Psi}, \mods{\bigwedge S})$\\[-0.25em]
\end{tabbing} 
\vspace{-1em}

\noindent  (GR$^{\scriptscriptstyle \circledast}$) is itself a special case of \CRevSn{2}.

Beyond these, they also propose two further possibly more controversial properties.  First, since they endorse \IndRevS~in the serial case, they argue in favour of a corresponding generalisation to the parallel case, which strengthens  \CRevRn{3}   and \CRevRn{4}, in the same way that \IndRevS~strengthened \CRevR{3} and \CRevR{4}. The formulation of this principle is obvious:  

\begin{tabbing}
\=BLAHBL\=\kill

\> \IndRevSn \> If $S_1\cup \bel{\Psi\circledast S_2}$ is consistent,  then  $S_1\subseteq$\\
\> \> $  \bel{(\Psi\circledast S_1)\circledast S_2} $\\[0.1cm]

\> \IndRevRn \>  If $x \in\mods{\bigwedge S}$, $y \notin\mods{\bigwedge S}$ and $x \preccurlyeq_{\Psi} y$, then \\
\> \>  $x \prec_{\Psi \circledast S}   y$\\[-0.25em]
\end{tabbing} 
\vspace{-1em}


\noindent Finally, they propose:

\begin{tabbing}
\=BLAHBL\=\kill

\> \PeeRev \>  If $S_1\cup S_2$ is consistent, then  $S_1\subseteq$\\
\> \> $  \bel{(\Psi\circledast (S_1\cup \neg S_2))\circledast S_2}$\\[0.1em]

%
%

\> ($\mathrm{P}^{\scriptscriptstyle \circledast}_{\scriptscriptstyle \min}$) \>    If $S_1\cup S_2$ is consistent,  then \\
\> \> $ \min(\preccurlyeq_{\Psi\circledast (S_1\cup \neg S_2)}, \mods{\bigwedge S_2})\subseteq \mods{S_1}$\footnotemark{}   \\[-0.25em]

\end{tabbing} 
\vspace{-1em}

\footnotetext{\PeeRev~and  \EssRev~are formulated in a different, but logically equivalent, manner in the original text.}

%
%
%
%

\noindent Clearly, \PeeRev~aims to deliver the right kind of intuition in relation to cases like Example \ref{ex:KPPRevise}. It tells us that, if we revise by a set of sentences $S$ and then by the negation of a subset of these, the non-negated remainder of $S$ survives this second revision (as long as it is consistent with its input).


The postulate, however, seems unduly strong.
%
Indeed, restricted to singleton input sets, it gives us the following problematic consequence: if $A$ and $\neg B$ are consistent, then $A\in\bel{(\Psi\circledast\{A, B\})\circledast\{\neg B\}}$.
To see why this is an issue:

\begin{example}
\label{ex:KPPTwo}
As in Example \ref{ex:KPPRevise} but we are convinced that our friend Bill has made sure that the functionalities of the adder and multiplier are positively correlated, so that the adder is working  if and only if  the multiplier is working too.

\end{example}

\noindent  Although $A$ and $\neg B$ are jointly logically consistent, it seems incorrect to then hold that $A\in\bel{(\Psi\circledast\{A, B\})\circledast\{\neg B\}}$.

Delgrande \& Jin show their properties are sound for a particular construction. This construction, however, is  both  by their own admission, ``somewhat complicated'' and based on Spohn's ranking function formalism (see \cite{Spohn1988-SPOOCF}), whose recent axiomatic foundations remain poorly understood  \cite{ChandlerSpohn}.

In this section, we've seen that Delgrande \& Jin have convincingly argued against Zhang's strong generalisation of  Darwiche \& Pearl's second postulate  (Zhang's \CRevSnPlus{2}) and offered a plausible strengthening of his generalisations of the remaining three (his \CRevSn{1}, \CRevSn{3} and \CRevSn{4}), via \PCRevSn{3} and \PCRevSn{4}. They have also introduced a new promising postulate \EssRev, and, for those who find  \IndRevS~to be a compelling constraint in the serial case, an appropriate corresponding generalisation to the parallel case, in the form of  \IndRevSn. On the more negative side, their proposal does  include the implausible \PeeRev~and currently lacks a compelling constructive foundation.

Next, we propose a constructive approach that enables one to carry over, in a principled manner, constraints from  iterated serial  to iterated parallel revision. Under mild assumptions, it validates the more plausible principles that we have seen (e.g. \PCRevSn{3},  \PCRevSn{4}, and \EssRev), while invalidating the less plausible ones (e.g.~\CRevSnPlus{2} and \PeeRev). It  draws from the order aggregation-based approach to iterated parallel contraction of \cite{chandler2025parallelbeliefcontractionorder}. In the next section, we briefly recapitulate (i)  the  ``TeamQueue'' aggregation method presented there, which generalises the approach in \cite{DBLP:journals/ai/BoothC19}, and (ii) its application to the problem of iterated parallel contraction. 

\section{TeamQueue aggregation and iterated parallel contraction} 
\label{sec:TQ}


Where $I$ is an index set, a TeamQueue aggregator is a function $\oplus$ taking as inputs tuples $\mathbf{P} = \langle\preccurlyeq_{i}\rangle_{i\in I}$ of TPO over $W$  known as ``profiles'' and returning single TPOs over $W$ as outputs. When the identity of $\mathbf{P} $ is clear from context, we write $\oplus\mathbf{P}$ to denote  $\oplus (\mathbf{P})$ and  $x \preccurlyeq_{\oplus\mathbf{P}} y$ to denote $\langle x, y\rangle\in \oplus\mathbf{P}$, or simply $x \preccurlyeq_{\oplus} y$. The constructive definition makes use of the representation of a TPO $\preccurlyeq_j$ by means of an ordered partition $\langle S_1, S_2, \ldots S_{m_j}\rangle$ of $W$, defined inductively by setting, for each $i \geq 1$, $S_i=\min(\preccurlyeq_j, \bigcap_{k<i}S^c_{k})$, where $S^c$ is the complement of $S$. This representation grounds the notion of  the {\em absolute} rank $r_j(x)$ of an alternative $x$, with respect to $\preccurlyeq_j$. The absolute rank of an alternative is given by its position in the ordered partition, so that $r_j(x)$ is such that $x\in S_{r_j(x)}$. The aggregation method is then defined inductively as follows:

\begin{definition}
$\oplus$ is a {\em TeamQueue (TQ) aggregator} iff, for each profile $\mathbf{P}=\langle\preccurlyeq_1, \ldots \preccurlyeq_n\rangle$, there exists a sequence $\langle a_{\mathbf{P}}(i)\rangle_{i \in \mathbb{N}}$ such that $\emptyset \neq a_{\mathbf{P}}(i) \subseteq \{ 1, \ldots, n\}$ for each $i$  and the ordered partition $\langle T_1, T_2, \ldots, T_m\rangle$ of indifferences classes corresponding to $\preccurlyeq_{\oplus}$ is constructed inductively as follows:
\[
T_{i} = \bigcup_{j \in a_{\mathbf{P}}(i)} \min(\preccurlyeq_j, \bigcap_{k<i}T_{k}^c)
\]
where  $m$ is minimal s.t. $\bigcup_{i\leq m} T_i = W$. 
\end{definition}

\noindent In the initial step,  this procedure cuts the minimal elements from one or more of the input TPOs and then pastes them into the lowest rank of the output TPO. Any remaining duplicates of these minimal elements in the input TPOs are also deleted.  This operation is reiterated at each subsequent step until all input TPOs have been fully processed.

One noteworthy member of the TQ aggregator family manages the queues in a ``synchronous'' or again ``concurrent'' manner. This implies that at every stage, the minimal elements from {\em all} TPOs are incorporated into the appropriate output rank. This aggregator is defined as follows:

\begin{definition}
The {\em Synchronous TeamQueue (STQ)} aggregator $\STQ$ is the TeamQueue aggregator for which $a_{\mathbf{P}}(i) = \{1,\ldots, n\}$ for all profiles $\mathbf{P}=\langle \preccurlyeq_1,\ldots,  \preccurlyeq_n \rangle$ and all $i$. 
\end{definition}

\noindent In \cite{chandler2025parallelbeliefcontractionorder}, the family of TQ aggregators was shown to be characterisable in terms of the following ``factoring'' property, which we give both in terms of minimal sets and in terms of binary relations:

\begin{tabbing}
\=BLAHBLAI\= ooooo \= \kill
 \> \FacMin \> For all $S\subseteq W$, there exists $X\subseteq I$ s.t.\\
\> \> $ \min(\preccurlyeq_{\oplus},  S )= \bigcup_{j\in X} \min(\preccurlyeq_j,  S )$ \>\\[0.1cm]
\> \FacPref \> Assume that $x_1$ to $x_n$ are s.t.~$x_i \preccurlyeq_i y$. Then  \\
\> \>  there  exists $j\in I$ s.t.~ \>  \\
\>  \>  (i)  \> if  $x_j \prec_j y$, then  $x_j \prec_\oplus y$, and \\[0.1cm]
\> \>  (ii)  \> if  $x_j \preccurlyeq_j y$, then  $x_j \preccurlyeq_\oplus y$\\[-0.25em]
\end{tabbing}

\vspace{-1em}

\noindent A characterisation of $\STQ$ can be achieved by supplementing \FacMin~ with a principle of ``Parity'':

%
%


\begin{tabbing}
\=BLAHBLAI\=\kill

\> \PPARMin    \>  If $x\prec_\oplus y $ for all $x \in S^c$, $y \in S$, then \\
\> \> $ \bigcup_{i\in I}\min( \preccurlyeq_{i},S)\subseteq \min( \preccurlyeq_{\oplus},S)$\\[0.1cm] 
\> \PPAR   \>  If $x \prec_{\oplus} y$ then for each $i \in I$ there exists $z$ s.t. \\
\> \> $ x \sim_{\oplus} z$ and  $z \prec_i y$\\[-0.25em]
\end{tabbing}
 \vspace{-1em}

 \noindent In \cite{chandler2025parallelbeliefcontractionorder}, it was shown that there exists a close connection between $\STQ$ and  the concept of the rational closure of a set of conditionals $\Gamma$ (\cite{lehmann1992does}), which represents the rational set of conditionals that can be inferred from $\Gamma$. More specifically, it was shown that $\cbel{\preccurlyeq_{ \STQ} }=\CRat(\bigcap_i \cbel{\preccurlyeq_i})$.

\begin{figure}
  \centering
    \includegraphics[width=0.35\textwidth]{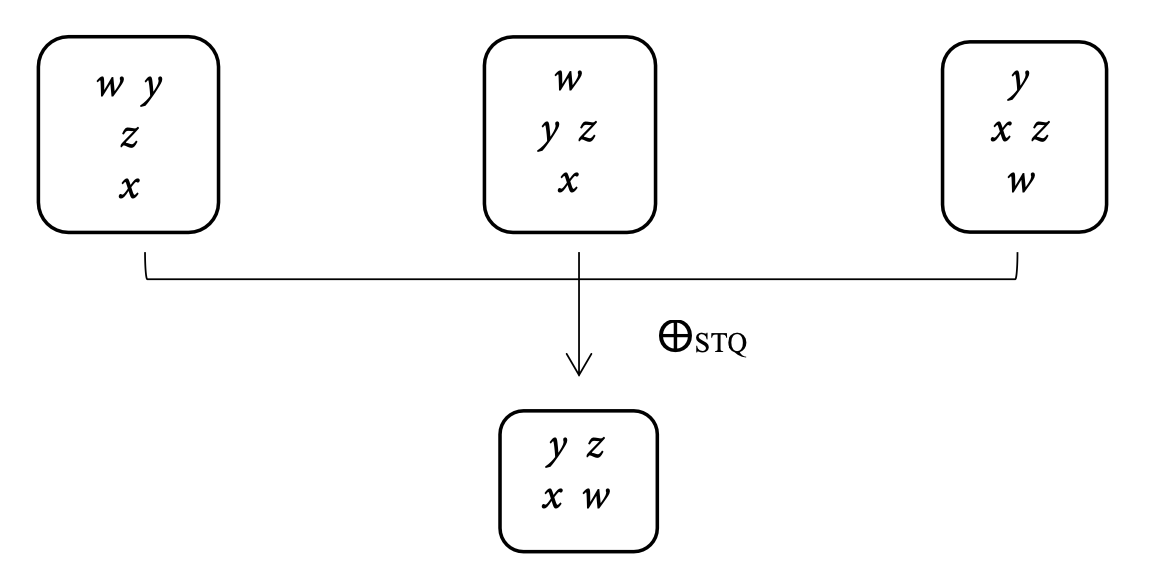}
  \caption{Illustration of the $\STQ$ aggregator, using 3 inputs.  Boxes represent TPOs, with lower case letters arranged such that a lower letter corresponds to a lower world in the relevant ordering. 
  }
  \label{fig:TQeg}
\end{figure}

In the same paper, Chandler \& Booth made use of TeamQueue aggregation  to  define iterated parallel contraction in terms of iterated serial contraction. They assume the AGM postulates for serial contraction, as well as the following analogues for serial contraction of the DP postulates proposed by Chopra and colleagues  \cite{chopra2008iterated}, given syntactically by: 
\begin{tabbing}
\=BLAHBL\=\kill

\> \CConS{1}  \>  If $\neg A \in \mbox{Cn}(B)$ then $\bel{(\Psi \contract A) \ast B} = \bel{\Psi \ast B}$ \\[0.1cm]

\> \CConS{2} \> If $A \in \mbox{Cn}(B)$ then $\bel{(\Psi \contract A) \ast B} = \bel{\Psi \ast B}$ \\[0.1cm]

\> \CConS{3}   \> If $\neg A \in \bel{\Psi \ast B}$ then $\neg A \in \bel{(\Psi \contract A) \ast B}$ \\[0.1cm]

\> \CConS{4}  \> $A \not\in \bel{\Psi \ast B}$ then $A \not\in \bel{(\Psi \contract A) \ast B}$\\[-0.25em]
\end{tabbing} 
\vspace{-1em}
\noindent and semantically by:
\begin{tabbing}
\=BLAHBL\=\kill

\> \CConR{1}  \>  If $x,y \in \mods{\neg A}$ then $x \preccurlyeq_{\Psi \contract A} y$ iff  $x \preccurlyeq_\Psi y$ \\[0.1cm]

\> \CConR{2}  \> If $x,y \in \mods{A}$ then $x \preccurlyeq_{\Psi \contract A} y$ iff  $x \preccurlyeq_\Psi y$\\[0.1cm]

\> \CConR{3}   \> If $x \in \mods{\neg A}$, $y \in \mods{A}$ and $x \prec_\Psi y$ then \\
\> \> $ x \prec_{\Psi \contract A} y$ \\[0.1cm]

\> \CConR{4} \> If $x \in \mods{\neg A}$, $y \in \mods{A}$ and $x \preccurlyeq_\Psi y$ then \\
\> \> $ x \preccurlyeq_{\Psi \contract A} y$\\[-0.25em]
\end{tabbing} 
\vspace{-1em}

\noindent They endorse the following principle:
\begin{tabbing}
\=BLAHBL\=\kill

\> \MultiConAggregPrec\> $\preccurlyeq_{\Psi\odiv\{A_1,\ldots, A_n\}}=\oplus\{\preccurlyeq_{\Psi\contract A_1}, \ldots, \preccurlyeq_{\Psi\contract A_n}\}$\\[-0.25em]
\end{tabbing} 
\vspace{-1em}

\noindent and  offer the following definition 

\begin{definition}
\label{def:TQPackageContraction}
$\odiv$ is a {\em TeamQueue (resp.~Synchronous  TeamQueue) parallel contraction operator} if and only if it is a parallel contraction operator such that there exists a serial contraction operator $\contract$ satisfying the AGM postulates and the postulates of Chopra {\em et al} and a TeamQueue aggregator $\oplus$ (resp.~Synchronous TeamQueue aggregator $\STQ$), such that, for all $\Psi$ and $S\subseteq L$, $\preccurlyeq_{\Psi\odiv S}$ is  defined from the $\preccurlyeq_{\Psi\contract A_i}$ by \MultiConAggregPrec, using $\oplus$ (resp.~$\STQ$). \end{definition}

\noindent If we impose the constraint that $a_{\mathbf{P}}(1) = \{1,\ldots, n\}$ on the construction of $\oplus$, as is the case in relation to $\STQ$, then TeamQueue  parallel contraction yield the ``intersective'' definition of single-step parallel contraction, the principle according to which the belief set obtained after contraction by a set $S$ is given by the intersection of the belief sets obtained after contractions by each of the members of $S$ ($[\Psi\odiv\{A_1,\ldots, A_n\}]=\bigcap_{1\leq i\leq n}[\Psi\contract A_i]$). This intersective definition was proven in \cite{chandler2025parallelbeliefcontractionorder} to entail a set of appealing generalisations to the parallel case of the AGM postulates for serial contraction.

In \cite{chandler2025parallelbeliefcontractionorder},  it was also shown that the TQ approach allows us to recover generalisations to the parallel case of  the postulates of Chopra {\em et al}. These are:
\begin{tabbing}
\=BLAHBL\=\kill

\> \CConRn{1}  \>  If $x,y \in\mods{\bigwedge \neg S}$ then $x \preccurlyeq_{\Psi \odiv S}   y$ iff  $x \preccurlyeq_\Psi y$ \\[0.1cm]

\> \CConRn{2} \> If $x,y \in\mods{\bigwedge S}$ then $x \preccurlyeq_{\Psi \odiv S}   y$ iff  $x \preccurlyeq_\Psi y$\\[0.1cm]

\> \CConRn{3}   \> If $x \in\mods{\bigwedge \neg S}$, $y \notin\mods{\bigwedge \neg S}$  and $x \prec_\Psi y$ then \\
\> \> $ x \prec_{\Psi \odiv S}   y$ \\[0.1cm]

\> \CConRn{4} \> If $x \in\mods{\bigwedge \neg S}$, $y \notin\mods{\bigwedge \neg S}$ and $x \preccurlyeq_\Psi y$ then \\
\> \> $ x \preccurlyeq_{\Psi \odiv S}   y$\\[-0.25em]
\end{tabbing} 
\vspace{-1em}

\noindent Syntactic counterparts for these are provided as follows:

\begin{tabbing}
\=BLAHBL\=\kill

\> \CConSn{1}\> If $\neg S_1\subseteq \Cn{(S_2)}$, then  $\bel{(\Psi\odiv S_1)\circledast S_2}$\\
\> \> $ = \bel{\Psi\circledast S_2}$  \\[0.1cm]

\> \CConSn{2} \> If $S_1\subseteq \Cn{(S_2)}$, then  $\bel{(\Psi\odiv S_1)\circledast S_2}$\\
\> \> $  = \bel{\Psi\circledast S_2}$  \\[0.1cm]

\> \CConSn{3}   \> If $\neg S_1\subseteq \bel{\Psi\circledast S_2}$, then  $\neg S_1\subseteq$$  \bel{(\Psi\odiv S_1)\circledast S_2} $ \\[0.1cm]

\> \CConSn{4} \> If $\neg S_1\cup \bel{\Psi\circledast S_2}$ is consistent, then\\
\> \>   $\neg S_1\cup \bel{(\Psi\odiv S_1)\circledast S_2} $ is consistent \\[-0.25em]

\end{tabbing} 
\vspace{-1em}

\section{Parallel Revision via TeamQueue aggregation} 
\label{sec:AggRev}


%

TQ aggregation also offers a promising, though slightly less direct, route to parallel revision. This approach requires two steps:
\begin{itemize}

\item[(i)] TQ aggregate the TPOs obtained from the individual revisions by the members of $S$, yielding $\oplus\{\preccurlyeq_{\Psi\ast A_1},\ldots, \preccurlyeq_{\Psi\ast A_n}\}$,

\item[(ii)] Transform this TPO to ensure that the ``Success'' postulate  \KRevP{2} is satisfied, i.e.~that $S\subseteq\bel{\Psi\circledast S}$.


\end{itemize}
This second step is required, since:

\begin{restatable}{prop}{StepTwoRequired}
\label{prop:StepTwoRequired}
There exists a set of sentences $S\subseteq L$, an AGM and DP serial belief revision operator $\ast$ and a state $\Psi$, s.t.~ $\min(\oplus\{\preccurlyeq_{\Psi\ast A_1},\ldots, \preccurlyeq_{\Psi\ast A_n}\}, W)\nsubseteq\mods{\bigwedge S}$.
\end{restatable}

\noindent  In the special case in which the single sentence revision operator $\ast$ is one that identifies belief states with TPOs (such as the natural, lexicographic or again restrained revision operators), the obvious choice of transformation in step (ii) would simply be a serial revision by $\bigwedge S$. We could take this revision to use the same serial revision operator as the one used in the first step. Alternatively, we could choose the natural revision operator $\astN$, which has some attractive properties and  involves the absolute minimal change required to get the job done.
Either way, the proposal would schematically be:


\begin{tabbing}
\=BLAHBL\=\kill

\>  \MultiRevAggregPrec   \>  $\preccurlyeq_{\Psi\circledast S}= \oplus\{\preccurlyeq_{\Psi\ast A_1},\ldots, \preccurlyeq_{\Psi\ast A_n}\}\ast' \bigwedge S$ \\[-0.25em]

\end{tabbing} 
\vspace{-1em}

\noindent  This leads us to the following definition:

\begin{definition}
\label{def:TQPackageRevision}
$\circledast$ is a {\em TeamQueue (TQ; resp.~Synchronous  TeamQueue) parallel revision operator} if and only if it is a parallel revision operator such that there exists two AGM and DP revision operators $\ast$ and $\ast'$ and a TeamQueue aggregator $\oplus$ (resp.~Synchronous TeamQueue aggregator $\STQ$), such that, for all $\Psi$ and $S\subseteq L$, $\preccurlyeq_{\Psi\circledast S}$ is  defined from $\ast$ and $\ast'$ by \MultiRevAggregPrec, using $\oplus$ (resp.~$\STQ$). \end{definition}
%
%
%

%
%
%

\noindent To illustrate how this approach might work, we depict in Fig. \ref{KPPModel} a plausible model of Example  \ref{ex:KPPRevise}. We note that it gives us  the correct intuitive outcome: we find that $A, B\notin\bel{\Psi}$ and $A\in\bel{(\Psi\circledast \{A, B\})\circledast\{\neg B\}}$, as required. The immediate upshot of this illustration, of course, is:

\begin{figure}
  \centering
    \includegraphics[width=0.3\textwidth]{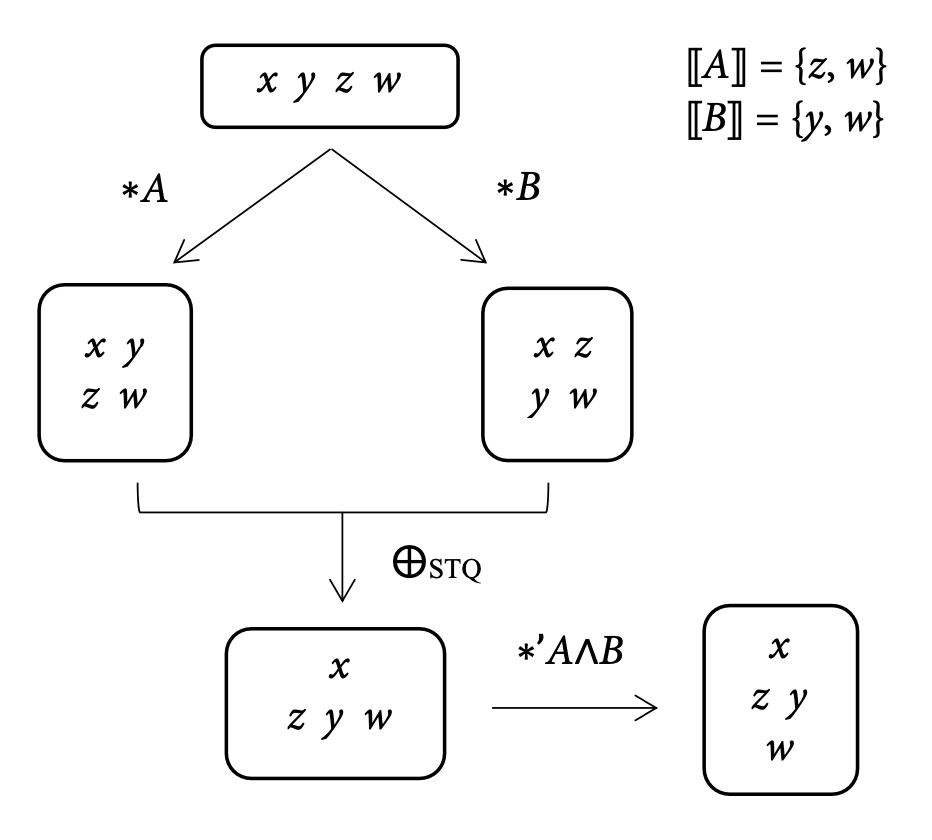}
  \caption{Model of Example \ref{ex:KPPRevise}, using  the TQ approach. 
  The correct result is obtained here: $A, B\notin\bel{\Psi}$ and $A\in\bel{(\Psi\circledast \{A, B\})\circledast \{\neg B\}}$, since $\min(\preccurlyeq_{\Psi\circledast \{A, B\}}, \mods{\neg B})= \{z\}\subseteq\{z, w\}=\mods{A}$.}
  \label{KPPModel}
\end{figure}

\begin{restatable}{prop}{NoStrongTwo}
\label{prop:NoStrongTwo}
 \CRevRnPlus{2} fails for  (even Synchronous) TQ parallel revision operators.
\end{restatable}

\noindent What then of the other properties we previously discussed? First, it yields \MultiRevInter~as its single-step special case:

\begin{restatable}{prop}{AggRedrev}
\label{AggRedrev}
If  $\circledast$ is  a TQ parallel revision operator, then it satisfies \MultiRevInter. 
\end{restatable}

\noindent Second, we recover Delgrande \& Jin's plausible strong principles  \PCRevRn{3} and  \PCRevRn{4}: 

\begin{restatable}{prop}{PCR}
\label{prop:PCR}
If  $\circledast$ is  a TQ parallel revision operator, then it satisfies  \PCRevRn{3} and   \PCRevRn{4}. 
\end{restatable}

\noindent We've seen that \PCRevRn{3} and \PCRevRn{4} jointly entail \CRevRn{1}--\CRevRn{4}. Proposition \ref{prop:PCR} therefore gives us:

\begin{restatable}{lem}{RtoRPack}
\label{lem:RtoRPack}
If  $\circledast$ is  a TQ parallel revision operator, then it satisfies \CRevRn{1}, \CRevRn{2}, \CRevRn{3} and \CRevRn{4}.

\end{restatable}

\noindent We can also recover the parallel version of  \IndRevR, on the assumption that we are proceeding from serial revision operators that satisfy \IndRevR:

\begin{restatable}{prop}{indep}
\label{prop:Indep}
Let  $\circledast$ be a TQ parallel revision operator defined from AGM and DP serial revision operators $\ast$ and $\ast'$. Then, if $\ast$ and $\ast'$ satisfy  \IndRevR, then $\circledast$ satisfies  \IndRevRn.
\end{restatable}

\noindent Finally, another key principle can also be recovered:

\begin{restatable}{prop}{SSoundForTQ}
\label{prop:SSoundForTQ}
If $\circledast$ is a TQ parallel revision operator,  then it satisfies \EssRev.
\end{restatable}


\noindent Regarding the principles that {\em don't} hold, we've already seen that \CRevRnPlus{2}  fails. Thankfully, the same applies to \PeeRev:

\begin{restatable}{prop}{PFails}
\label{prop:PFails}
\PeeRev~fails for (even Synchronous) TQ parallel revision operators.
\end{restatable}

\section{Concluding comments}
\label{sec:Concl}

Order aggregation, particularly when implemented using Booth \& Chandler $\STQ$ operator, provides a fruitful approach to iterated parallel revision. It yields a principled way to construct  revision operators that satisfy the plausible postulates proposed in \cite{DelgrandeJames2012PbrR},  without validating more questionable ones. Using this approach to extend a serial revision operator that identifies epistemic states with TPOs (e.g. the natural, restrained or lexicographic revision operators) covers indefinitely many iterations of revision. 

In terms of related future research, the most obvious results to seek would be {\em characterisations} of various noteworthy classes of TQ parallel revision operators, based on different interesting classes of serial revision operators, such as, for example, the class of serial revision operators satisfying both the AGM and the DP postulates. At this stage, we only have a number of {\em soundness} results.

Secondly,  as we have seen in Section \ref{sec:PrincBelCh}, even in relation to the {\em single-step} parallel case, the most obvious generalisations of the Harper and Levi identities, i.e. \HIP~and \LIP, are not promising. The situation is even less clear in relation to the iterated parallel case and work on iterated versions of \HI~and \LI~for serial change is a fairly recent development (see for instance \cite{DBLP:conf/dagstuhl/NayakGOP05},  \cite{DBLP:conf/lori/Chandler019}, and \cite{DBLP:journals/ai/BoothC19}).  Nevertheless, it would be interesting to find an elegant way of connecting the respective extensions to the parallel case of the AGM postulates for serial revision and contraction, be this in connection with single-step or iterated change.


%

\bibliographystyle{splncs04}
\bibliography{CHANDLER_biblio}

\vfill
\pagebreak{}


\section*{Appendix: proofs}\label{s:appendix}


\RPackSyntPlus*

\begin{pproof}
From \CRevRn{2} to \CRevSn{2}: Assume that $\neg S_1\subseteq \Cn{(S_2)}$. We need to show that $\bel{(\Psi\circledast S_1)\circledast S_2} = \bel{\Psi\circledast S_2}$, i.e.~by \MultiRevInter~that $\min(\preccurlyeq_{\Psi\circledast S_1}, \mods{\bigwedge S_2}) = \min(\preccurlyeq_{\Psi}, \mods{\bigwedge S_2}) $. Regarding the corresponding left-to-right inclusion, suppose for reductio that there exists $x\in W$ such that $x\in \min(\preccurlyeq_{\Psi\circledast S_1}, \mods{\bigwedge S_2})$ but $x\notin \min(\preccurlyeq_{\Psi}, \mods{\bigwedge S_2})$. 
Then there exists $y\in \mods{\bigwedge S_2}$ such that $y\prec_{\Psi} x$. Since $x, y\in \mods{\bigwedge S_2}$ and $\neg S_1\in \Cn{(S_2)}$, it follows that $x, y\in \mods{\bigwedge \neg S_1}$. By \CRevRn{2}, it then follows that $y\prec_{\Psi\circledast S_1} x$, contradicting the assumption that $x\in \min(\preccurlyeq_{\Psi\circledast S_1}, \mods{\bigwedge S_2})$. The right-to-left  inclusion is then established in a similar way.

From  \CRevSn{2} to \CRevRn{2}: Assume that $x,y \in\mods{\bigwedge \neg S}$ and let $C\in L$ be such that $\mods{C}=\{x, y\}$. If $x\preccurlyeq_{\Psi} y$ but $y\prec_{\Psi\circledast S_1} x$, then $x\in \min(\preccurlyeq_{\Psi}, \mods{C})$ but $x\notin \min(\preccurlyeq_{\Psi\circledast S_1}, \mods{C})$. Since we have $\bigwedge \neg S\in\Cn(C)$, this contradicts \CRevSn{2}. The reasoning is similar if we instead assume  $y\prec_{\Psi} x$ but $x\preccurlyeq_{\Psi\circledast S_1} y$. Hence $x\preccurlyeq_{\Psi} y$ iff $x\preccurlyeq_{\Psi\circledast S_1} y$, as required
\end{pproof}

\vspace{1em}


\StepTwoRequired*

\begin{pproof}
Let $W=\{x, y, z, w\}$, $\prec_\Psi $ be given by $x\prec_\Psi \{y, z, w\} $, with $\mods{A}= \{z, w\}$ and $\mods{B}= \{y, w\}$. Then we have $\min(\STQ(\preccurlyeq_{\Psi\ast A}, \preccurlyeq_{\Psi\ast B}), W) = \{z, y, w\} \nsubseteq \{w\}= \mods{A\wedge B}$.  
\end{pproof}

\vspace{1em}

\AggRedrev*

\begin{pproof}
We first note that, as was established in \cite{ChanBoothPCOA}, TeamQueue aggregators satisfy the following two principles:

\begin{tabbing}
\=BLAHBL\=\kill
\> \UBO \> For all $S\subseteq W$, $\min(\preccurlyeq_\oplus, S) \subseteq$\\
\> \> $  \bigcup_{i\in I} \min(\preccurlyeq_i, S)$ \\[0.1cm] 
\> \LBO   \> For all $S\subseteq W$, there exists $i\in I$  s.t.\\
\> \> $ \min(\preccurlyeq_i, S)\subseteq \min(\preccurlyeq_\oplus, S)  $ \\[-0.25em]
\end{tabbing} 
\vspace{-1em}

\noindent In its semantic form, \MultiRevInter~amounts to 
$\min(\preccurlyeq_{\Psi\circledast S}, W)= \min(\preccurlyeq_{\Psi}, \mods{\bigwedge S})$. Since $\ast$ satisfies AGM, \MultiRevAggregPrec ~entails that we have $\min(\preccurlyeq_{\Psi\circledast S}, W) = \min(\preccurlyeq_{\STQ}, \mods{\bigwedge S})$. So we now need to show that  $ \min(\preccurlyeq_{\STQ}, \mods{\bigwedge S})=\min(\preccurlyeq_{\Psi}, \mods{\bigwedge S}) $. By \CRevR{1}, for all $i\in I$,  $\min(\preccurlyeq_{\Psi\ast A_i}, \mods{\bigwedge S})=\min(\preccurlyeq_{\Psi}, \mods{\bigwedge S})$. By \UBO, we therefore have $ \min(\preccurlyeq_{\STQ}, \mods{\bigwedge S})\subseteq\min(\preccurlyeq_{\Psi}, \mods{\bigwedge S}) $. The converse inclusion holds by virtue of \LBO.
\end{pproof}

\vspace{1em}


\PCR*

\begin{pproof}
We first note that, as was established in \cite{ChanBoothPCOA}, TeamQueue aggregators satisfy the following two principles:

\begin{tabbing}
\=BLAHBLA\=\kill
\>   \SPU   \>  If,  for all $i\in I$, $x\prec_i  y$, then  $x\prec_\oplus y$  \\[0.1cm]
\>  \WPU   \>  If,  for all $i\in I$, $x\preccurlyeq_i  y$, then  $x\preccurlyeq_\oplus y$  \\[-0.25em]
\end{tabbing}
\vspace{-1em}

\noindent Regarding \PCRevRn{3}: Assume that $(S\!\mid\! y)\subseteq (S\!\mid\! x)$ and that $x \prec_\Psi y$. We have three possibilities to consider, for any $A_i\in S$:
\begin{itemize}

\item[(i)] Assume $A_i\in (S\!\mid\! x)$ and  $A_i\in (S\!\mid\! y)$, so that $x, y\in \mods{A_i}$. Then, since $x \prec_\Psi y$, by  \CRevR{1}, we have $x \prec_{\Psi\ast A_i} y$. 

\item[(ii)] Assume $A_i\in (S\!\mid\! x)$ and $A_i\notin (S\!\mid\! y)$, so that $x\in \mods{A_i}$ but $y\notin \mods{A_i}$. Then, again since $x \prec_\Psi y$, we recover $x \prec_{\Psi\ast A_i} y$, but this time using  \CRevR{3}.

\item[(iii)] Assume $A_i\notin (S\!\mid\! x)$ and $A_i\notin (S\!\mid\! y)$, so that $x, y\notin \mods{A_i}$. Then we obtain, $x \prec_{\Psi\ast A_i} y$, using   \CRevR{2}.

\end{itemize}
So we know that, for all $i\in I$,  $x \prec_{\Psi\ast A_i} y$. By \SPU, we then have $x \prec_{\oplus} y$. Since $(S\!\mid\! y)\subseteq (S\!\mid\! x)$, we know that if  $y\in \mods{\bigwedge S}$, then  $x\in \mods{\bigwedge S}$. It follows from this, by  \CRevR{1},  \CRevR{2},  \CRevR{3}, and  \MultiRevAggregPrec, that $x \prec_{\Psi \circledast S}   y$, as required.

Regarding \PCRevRn{4}, the proof is entirely analogous, using  \CRevR{4} in place of \CRevR{3} and \WPU~in place of \SPU.  
\end{pproof}

\vspace{1em}

\indep*

\begin{proof} 
Assume $x \in\mods{\bigwedge S}$, $y \notin\mods{\bigwedge S}$ and $x \preccurlyeq y$. We need to show that $x \prec_{\Psi \circledast S}   y$. For all $i$, either (i) $x, y\in\mods{A_i}$ or (ii) $x\in\mods{A_i}$ and $y\in\mods{\neg A_i}$.  If (i) is the case, then $x\preccurlyeq_{\Psi \ast A_i} y$, by  \CRevR{1}. If (ii) is the case, then  $x\prec_{\Psi \ast A_i} y$, by  \IndRevR. Hence, either way, $x\preccurlyeq_{\Psi \ast A_i} y$. It then follows by \WPU~that  $x\prec_{ \oplus\}\preccurlyeq_{\Psi\ast  A_1},\ldots, \preccurlyeq_{\Psi\ast A_n}\}} y$. Since $\ast'$ satisfies  \IndRevR, $x \in\mods{\bigwedge S}$, and $y \notin\mods{\bigwedge S}$, we then have  $x\prec_{ \oplus\{\preccurlyeq_{\Psi\ast  A_1},\ldots, \preccurlyeq_{\Psi\ast A_n}\}\ast'\bigwedge S} y$ and hence $x \prec_{\Psi \circledast S}   y$, as required.  
\end{proof}

\vspace{1em}


\SSoundForTQ*

\begin{pproof}
We will derive \EssRev~in its minimal set form $(\mathrm{S}^{\scriptscriptstyle \circledast}_{\scriptscriptstyle \min})$, i.e. $\min(\preccurlyeq_{\Psi}, \mods{\bigwedge(S_1\cup \neg S_2)})=\min(\preccurlyeq_{\Psi\circledast (S_1\cup S_2)}, \mods{\bigwedge (S_1\cup \neg S_2)})$. The proof is similar to that of Proposition \ref{prop:SSoundForTQContract}, but proceeds in three steps rather than two, owing to the additional post-aggregation  revision step involved in the construction: 
\begin{itemize}

\item[(1)] We first show that, for all $A\in S_1\cup S_2$, $\min(\preccurlyeq_{\Psi}, \mods{\bigwedge(S_1\cup \neg S_2)})=\min(\preccurlyeq_{\Psi\ast A}, \mods{\bigwedge (S_1\cup \neg S_2)})$. We have two cases to consider here:
\begin{itemize}

\item[(i)] Assume $A\in S_1$: Then $\mods{\bigwedge (S_1\cup \neg S_2)}\subseteq\mods{A}$ and it follows by \CRevR{1} that $\min(\preccurlyeq_{\Psi}, \mods{\bigwedge(S_1\cup \neg S_2)})=\min(\preccurlyeq_{\Psi\ast A}, \mods{\bigwedge (S_1\cup \neg S_2)})$, as required.

\item[(ii)] Assume $A\in S_2$: Then $\mods{\bigwedge (S_1\cup \neg S_2)}\subseteq\mods{\neg A}$ and the required equality follows by \CRevR{2}.

\end{itemize}
\item[(2)]  We now establish that $\min(\preccurlyeq_{\Psi}, \mods{\bigwedge(S_1\cup \neg S_2)})=\min(\preccurlyeq_{\oplus}, \mods{\bigwedge (S_1\cup \neg S_2)})$, where $\preccurlyeq_\oplus$ denotes $\oplus\{\preccurlyeq_{\Psi\ast A_1},\ldots, \preccurlyeq_{\Psi\ast A_n}\}$. Since $\oplus$ satisfies the Factoring property \FacMin, we know that there exists $X\subseteq I$ such that   $\min(\preccurlyeq_{\oplus},  \mods{\bigwedge(S_1\cup \neg S_2)} )= \bigcup_{j\in X} \min(\preccurlyeq_{\Psi\ast A_j},  \mods{\bigwedge(S_1\cup \neg S_2)} )$. We then recover the required result from this and the fact established in (1).

\item[(3)] We finally derive $\min(\preccurlyeq_{\Psi}, \mods{\bigwedge(S_1\cup \neg S_2)})=\min(\preccurlyeq_{\Psi\circledast (S_1\cup S_2)}, \mods{\bigwedge (S_1\cup \neg S_2)})$. This follows by \CRevR{2} from the result obtained in the previous step, since $\mods{\bigwedge (S_1\cup \neg S_2)}\cap \mods{\bigwedge (S_1\cup S_2)} =\varnothing$.

\end{itemize}
\end{pproof}

\vspace{1em}


%
%
%
%
%
%
%
%


\PFails*

\begin{pproof}
See Figure \ref{fig:PFails}, which depicts a model such that $\{A, \neg B\}$ is consistent but $\{A\}\nsubseteq \bel{(\Psi\circledast \{A, B\})\circledast \{\neg B\}}$. Indeed, in this model we have $\min(\preccurlyeq_{\Psi\circledast \{A, B\}}, \mods{\neg B})= \{x\}\nsubseteq\{z, w\}=\mods{A}$. This  model corresponds to Example  \ref{ex:KPPTwo}, which was adduced as an intuitive  counterexample to \PeeRev.
\end{pproof}

\begin{figure}
  \centering
    \includegraphics[width=0.325\textwidth]{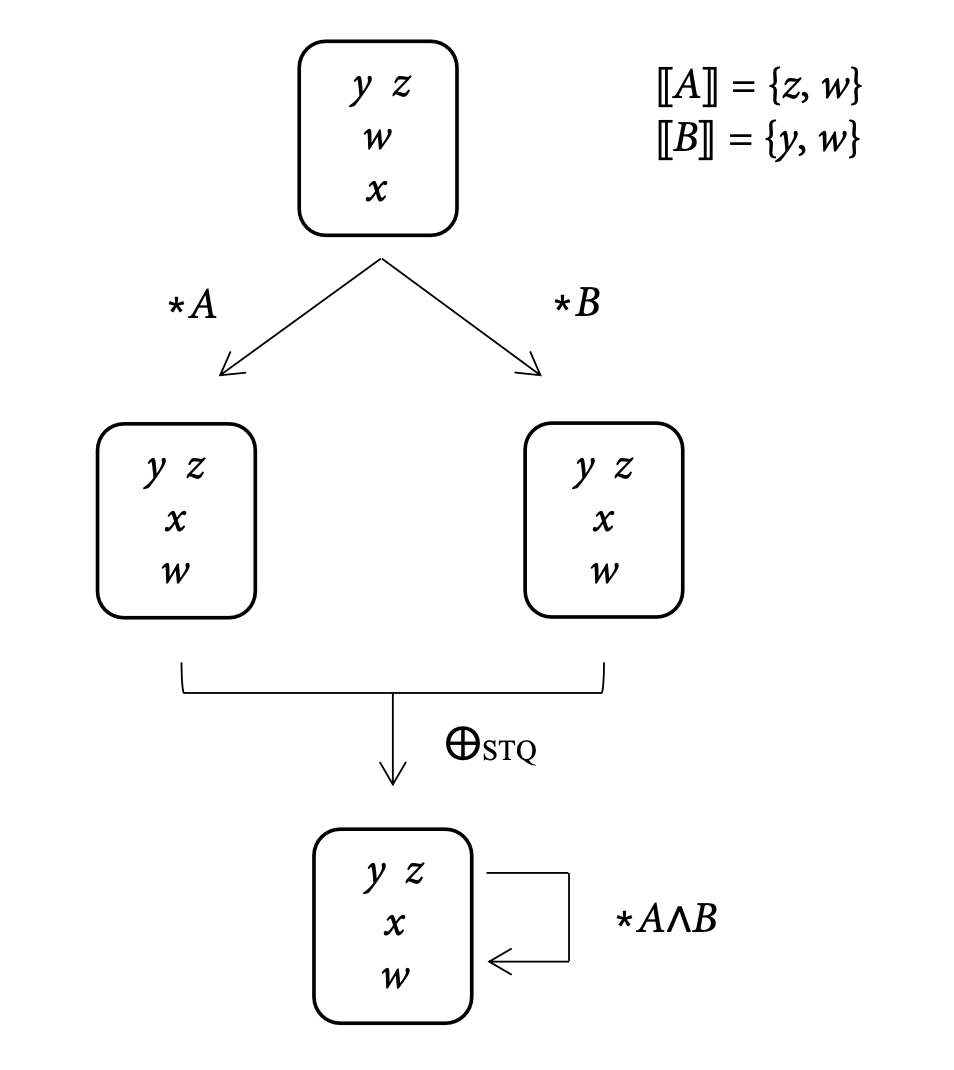}
  \caption{Illustration of the countermodel in the proof of Proposition \ref{prop:PFails}.}
  \label{fig:PFails}
\end{figure}

\vspace{1em}

\end{document}